\documentclass[conference]{IEEEtran}
\usepackage{graphicx}
\usepackage{amsthm}
\usepackage{epsfig}
\usepackage{latexsym}
\usepackage{amsfonts}
\usepackage{here}
\usepackage{rawfonts}
\usepackage[latin1]{inputenc}
\usepackage[T1]{fontenc}
\usepackage{calc}
\usepackage{capitalgreekitalic}
\usepackage{url}
\usepackage{enumerate}
\usepackage{color}
\usepackage[tbtags]{amsmath}
\usepackage{amssymb}
\usepackage{upref}
\usepackage{epic,eepic}
\usepackage{times}
\usepackage{dsfont}
\usepackage{comment}
\usepackage{cite}
\usepackage{bbm} 
\usepackage{optidef}
\usepackage{empheq}

\newtheorem{proposition}{\bf Proposition}

\usepackage{dsfont}

\newcounter{step}
\newlength{\totlinewidth}
\newenvironment{algorithm}{%
  \rule{\linewidth}{1pt}
  \begin{list}{}%
    {\usecounter{step}%
      \settowidth{\labelwidth}{\textbf{Step 2:}}%
      \setlength{\leftmargin}{\labelwidth}%
      \setlength{\topsep}{-2pt}%
      \addtolength{\leftmargin}{\labelsep}%
      \addtolength{\leftmargin}{2mm}%
      \setlength{\rightmargin}{2mm}%
      \setlength{\totlinewidth}{\linewidth}%
      \addtolength{\totlinewidth}{\leftmargin}%
      \addtolength{\totlinewidth}{\rightmargin}%
      \setlength{\parsep}{0mm}%
      \raggedright}}%
  {\end{list}%
  \rule{\linewidth}{1pt}}
\newcounter{substep}

  {\end{list}}

\newlength{\aligntop}
\newlength{\alignbot}
 \makeatletter
\renewenvironment{align}{%
  \vspace{\aligntop}
  \start@align\@ne\st@rredfalse\m@ne
}{%
  \math@cr \black@\totwidth@
  \egroup
  \ifingather@
    \restorealignstate@
    \egroup
    \nonumber
    \ifnum0=`{\fi\iffalse}\fi
  \else
    $$%
  \fi
  \ignorespacesafterend%
  \vspace{\alignbot}\par\noindent
} \makeatother

\IEEEoverridecommandlockouts

\usepackage{algorithm}
\usepackage{algorithmic}

\usepackage{subfigure}

\begin{document}
\title{Meta-Reinforcement Learning for Trajectory Design in Wireless UAV Networks}


\author{
\IEEEauthorblockN{Ye Hu\IEEEauthorrefmark{1},  Mingzhe Chen\IEEEauthorrefmark{2}\IEEEauthorrefmark{3}, Walid Saad\IEEEauthorrefmark{1}, H. Vincent Poor\IEEEauthorrefmark{2}, and Shuguang Cui\IEEEauthorrefmark{3}}  
\IEEEauthorblockA{\IEEEauthorrefmark{1}\small Wireless@VT, Bradley Department of Electrical and Computer Engineering, Virginia Tech, Blacksburg, VA, USA.\\
Email: \{yeh17,walids\}@vt.edu.} 
  \IEEEauthorblockA{\IEEEauthorrefmark{2}\small Department of Electrical Engineering, Princeton University, Princeton, NJ, USA.\\
Email:  \{mingzhec, poor\}@princeton.edu.} 
\IEEEauthorblockA{\IEEEauthorrefmark{3}\small The Future Network of Intelligence Institute, Chinese University of Hong Kong, Shenzhen, China.\\
Email: robert.cui@gmail.com.} 
\vspace{-1cm}

\thanks{This work was supported by the U.S. National Science Foundation under Grants CCF-0939370 and CCF-1908308.}}



\maketitle

\begin{abstract}
\boldmath
In this paper, the design of an optimal trajectory for an energy-constrained drone operating in dynamic network environments is studied. In the considered model, a drone base station (DBS) is dispatched to provide uplink connectivity to ground users whose demand is dynamic and unpredictable.
In this case, the DBS's trajectory must be adaptively adjusted to satisfy the dynamic user access requests.
 To this end, a meta-learning algorithm is proposed in order to adapt the DBS's trajectory when it encounters novel environments, by tuning a reinforcement learning (RL) solution. The meta-learning algorithm provides a solution that adapts the DBS in novel environments quickly based on limited former experiences.
 The meta-tuned RL is shown to yield a faster convergence to the optimal coverage in unseen environments with a considerably low computation complexity, compared to the baseline policy gradient algorithm. 
 Simulation results show that, the proposed meta-learning solution yields a $25\%$ improvement in the convergence speed, and about $10\%$ improvement in the DBS' communication performance, compared to a baseline policy gradient algorithm. Meanwhile, the probability that the DBS serves over $50\%$ of user requests increases about $27\%$, compared to the baseline policy gradient algorithm.
 
 
 \end{abstract}


\renewcommand{\thefootnote}{\fnsymbol{footnote}}




%
\IEEEpeerreviewmaketitle
 \vspace{-0.06cm}  
\section{Introduction}

Drones can potentially provide a cost-effective, flexible approach to boost the performance of wireless networks by providing services to hotspots, disaster-affected, or rural areas \cite{8755300, mozaffari2016unmanned}. However, effectively deploying drone base stations (DBSs) in a dynamic wireless environment is still challenging. 
In particular, designing a DBS trajectory that allows it to provide timely on-demand service is a major challenge, particularly when the ground users' requests are highly unpredictable.



The existing literature such as in \cite{7888557, 8531711, 9032206, 9037325} has studied a number of problems related to trajectory design for drone-based systems. The work in \cite{7888557} 
studies the trajectory optimization problem by jointly considering both the drone's communication throughput and energy consumption.
 In\cite{8531711}, the authors design the drone trajectory under practical communication connectivity constraint. The authors in \cite{9032206} propose a trajectory design that enhances physical layer security. The work in \cite{9037325} studies how the three-dimensional (3D) antenna radiation pattern and backhaul constraint affect a drone's 3D trajectory design.
  Despite their promising results, these existing works \cite{7888557, 8531711, 9032206, 9037325} do not consider the design of drone trajectory in a network where the service requests from ground users vary over time. 
 Meanwhile, prior works such as \cite{zhang2018predictive} and \cite{8727504} study the use of machine learning for drone trajectory optimization in dynamic environments. The authors in \cite{zhang2018predictive} use machine learning to deploy a drone to provide on-demand wireless service to cellular users. 
 The authors in \cite{8727504} propose a multi-agent Q-learning-based algorithm to design trajectories for multiple UAVs based on the prediction of users' mobility.
 However, the reinforcement learning (RL)-based solutions in \cite{zhang2018predictive} and \cite{8727504} are mostly over-fitted to certain environments, and, thus, they cannot rapidly cope with new or highly-dynamic environments. In particular, the internal dynamics of the RL in \cite{zhang2018predictive} and \cite{8727504}, such as hyper-parameters, exploration strategies, and loss functions, are manually pre-determined to help the algorithms to design trajectories in specific environments. As such, even the slightest changes in the environment could introduce significant negative influence on the RL performance. In contrast, a meta-learning based trajectory design solution \cite{vanschoren2018meta} that automatically tunes the RL internal dynamics would be more appropriate for dynamic wireless environments. Meta-learning, also known as ``learning to learn'', provides a solution that adapt the DBS to new environments rapidly with only a few training examples.


The main contribution of this paper is, thus, a novel meta-learning approach for optimizing the trajectory of a DBS while considering the uncertainty and dynamic of terrestrial users' service request. In particular, we consider a trajectory design problem that enables a DBS to effectively provide on-demand uplink service to users and solve it with a meta-reinforcement learning based algorithm. While prior works such as \cite{park2019meta,8761319,8723067} used meta-learning to study various wireless problems, those works have not considered the DBS trajectory design problem. In contrast, here, we introduce a learning solution tailored to the drone-aided system whose goal is to  provide on time service to ground users.
Unlike previous meta-learning solutions such as \cite{munir2020multi,8761319,8723067}, which aim at performing prediction with less training examples, we propose a novel, lightweight meta-learning approach that adapts the DBS to unseen environments. The proposed meta-learning based solution tunes the hyper-parameters of RL algorithm with online cross validation. An approximation on the hyper-parameter update direction is also proposed to reduce the computational complexity.
%
Simulation results show that, by tuning the hyper-parameters, the proposed meta-learning solution yields a $25\%$ improvement in the convergence speed, and about $10\%$ improvement in the DBS' communication performance, compared to a vanilla policy gradient algorithm. With the proposed algorithm, the probability that the DBS serves over $50\%$ of user requests increases about $27\%$. The computational complexity of the proposed algorithm is also shown to be similar to that of the vanilla policy gradient algorithm.

The rest of this paper is organized as follows. The system model and problem formulation are described in Section \uppercase\expandafter{\romannumeral2}. In Section \uppercase\expandafter{\romannumeral3}, the proposed algorithm is discussed. In Section \uppercase\expandafter{\romannumeral4}, numerical simulation results are presented. Finally, conclusions are drawn in Section \uppercase\expandafter{\romannumeral5}.

\begin{figure}
\setlength{\belowcaptionskip}{-2pt}
\setlength{\abovecaptionskip}{-2pt} 
  \centering
  \includegraphics[width=6.6cm]{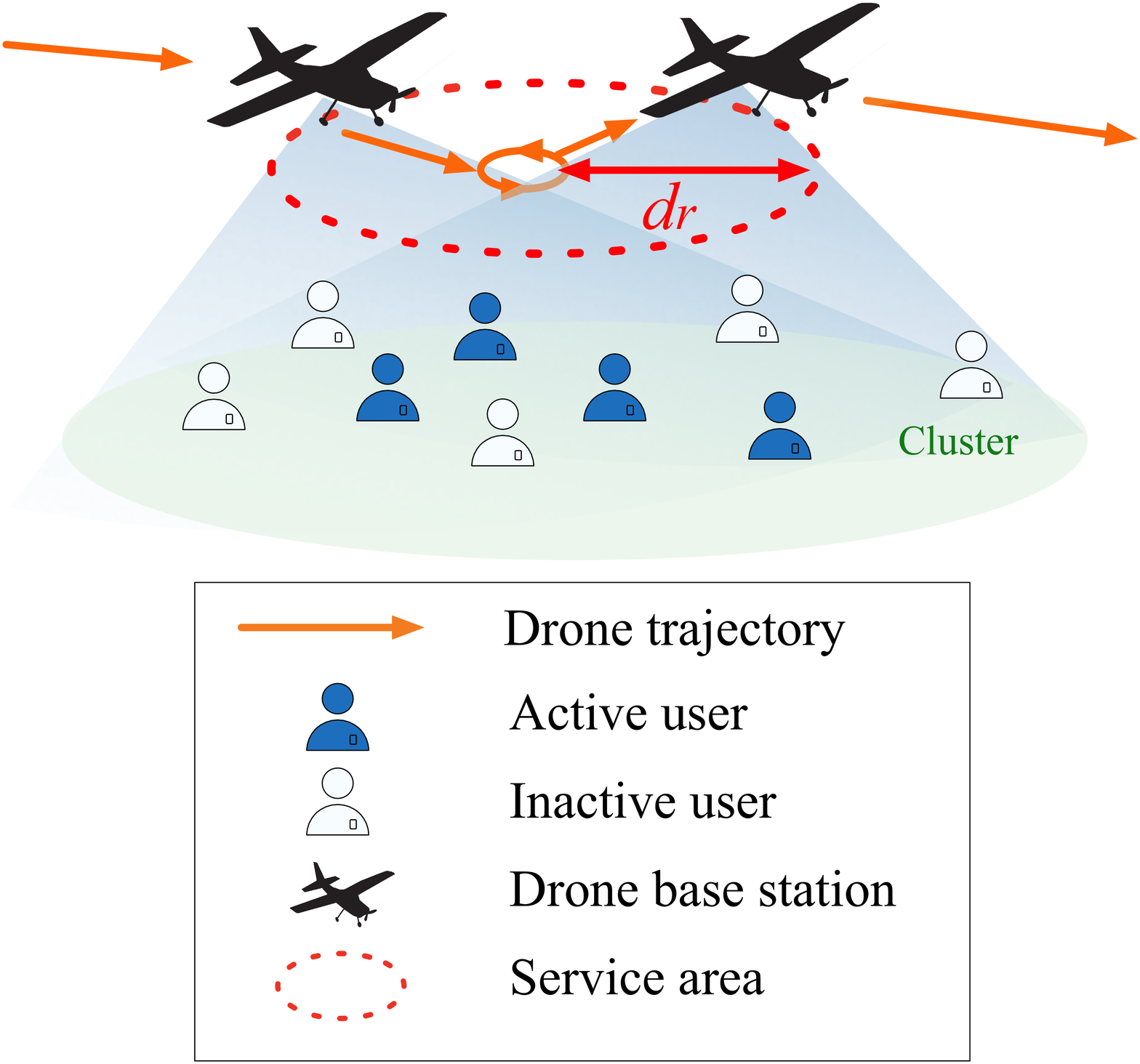}
  \caption{\footnotesize{Network topology.}}
  \label{Fig. 1}
  \centering
  \vspace{-0.5cm}
\end{figure}

\vspace{-0.12cm}

\section{System Model and Problem Formulation}

Consider a geographical area within which a set $\mathcal{U}$ of  $U$ randomly deployed terrestrial users request uplink data services. A fixed-wing DBS is dispatched to satisfy the uplink access requests of the terrestrial users, as shown in Fig. \ref{Fig. 1}. In such an area, a given user that is requesting data services is called an \emph{active user}, otherwise, it is called an \emph{inactive user}. We assume the users to be separated into several groups, each of which is called a \emph{cluster}. The set of all possible clusters is denoted as $\mathcal{L}$. The DBS will travel among the groups in a steady straight-and-level flight (SLF), and it uses a circle to hover over each group,  in a steady circular flight, with constant speed $V$ and constant altitude $H$ \cite{7888557}. Meanwhile, the DBS must return to its original location $O$ within a time period $T$ for battery charging. 
The trajectory that traces the DBS's movement within the studied time duration is captured by a vector $\boldsymbol{\xi}=\left[{l}_{1},\cdots,{l}_{K}\right]^\top$, with ${l}_{k}\in\mathcal{L}$ being the $k$-th cluster that the DBS serves and $K$ being the number of served clusters along the trajectory. The set of all possible trajectories of the DBS is given by $\mathcal{E}$. Let $\pi\left(l\left| l_{k},\tau_{k} \right.\right)$ be the probability that the DBS moves toward cluster $l\in\mathcal{L}$ after arriving at $l_{k}$ at time epoch $T-\tau_{k}$, in which $\tau_{k}$ is the remaining time for the DBS to return to the origin, after serving cluster $l_{k}$.

\vspace{-0.12cm}
\subsection{Communication Performance Analysis}
In our network, the users adopt an orthogonal frequency division multiple access (OFDMA) technique and transmit data over a set of uplink \emph{resource blocks} (RBs). The DBS will arbitrarily allocate one RB to each one of its associated users within a cluster. We assume that the DBS can keep serving its associated users within a $d_r$-meter radius over each cluster as shown in Fig.  \ref{Fig. 1}. The area within this range is called a \emph{service area}.
Also, user $u$ is assumed to request a total of $b _u$ data (in bits) at time epoch $t_u$. The DBS must satisfy the user's request within the studied duration ${T}$.  Let $\boldsymbol{b}=\left[b_1,\cdots,b_U\right]$ and $\boldsymbol{t}=\left[t_1,\cdots,t_U\right]$ be, respectively, the vector of quantity and occurrence of the users' access request in the network. Here, the quantity $b _u$ and active time $t_u$ are assumed to be independent random variables that follow unknown distributions. Every possible $\boldsymbol{b}$ and $\boldsymbol{t}$ that could be observed in this network is called one \emph{realization} of the users' access request.
In this model, the communication performance variation caused by the DBS's movement within the service area is considered to be negligible,  as the  variation on the DBS-user distance is considerably small with $d_r$ being much smaller than $H$.
The signal-to-noise ratio (SNR) at the link between the DBS and user $u$, will thus be:

 \begin{equation}\label{eq:1}
 \setlength{\abovedisplayskip}{-2 pt}
\begin{split}
&\gamma_{u}=\frac{Ph_{u,\phi_u}}{\sigma^2},
\end{split}
\end{equation} 
where $P$ represents the transmit power of user $u$, which is assumed to be equal for all users. $h_{u,\phi_u}$ is the path loss between user $u$ and the DBS, with $\phi_u$ is the resource block allocated to user $u$. Here, the Nakagami channel model is applied to characterize a wide range of fading environments for the link between user $u$ and the DBS, such that $h_{u,\phi_u}=\varepsilon d_{u}^{-\alpha}$ with $\varepsilon $ following Gamma distribution with shape parameter $m$, $d_u$ being the distance between user $u$ and the DBS, and $\alpha$ being the path loss exponent. 
The data rate at the link between the DBS and user $u$ will then be $c_{u}= {B}\log \left( {1 + \gamma_{u}} \right)$ with $B$ being the RB bandwidth (equal for all RBs). 

 \vspace{-0.12cm}
\subsection{Utility Function Model}
In the studied network, the transmission delay of user $u\in\mathcal{U}$,  when being served, is $D_{u}=\frac{b_u}{c_{u}}$. In this case, the time that DBS consumes to hover over cluster $l$ is given by: 
 \begin{equation}\label{eq:2}
  \setlength{\abovedisplayskip}{3 pt}
\setlength{\belowdisplayskip}{3 pt}
\begin{split}
 D^*_{l}=\max_{u\in \mathcal{U}_l} D_{u}-\frac{2d_r}{V},
\end{split}
\end{equation}
where $\mathcal{U}_l$ is the set of active users within cluster $l$. Note that, here, $\max_{u\in \mathcal{U}_l} D_{u}$ is the time the DBS consumes within the service area of cluster $l$. $D^*_{l}$ and $\frac{2d_r}{V}$ are the time the DBS consumes for hovering, and SLF, respectively, within the service area. Meanwhile, the DBS has to return to the origin within the considered, fixed time duration $T$, that is $t_{0}\left(\boldsymbol{\xi}\right)=\sum_{k=0}^{K}\frac{d_{l_{k},l_{k+1}}}{V}+\sum_{k=1}^{K}D^*_{l_{k}}\le T$, where $t_{0}$ is the time when the DBS returns to origin $O$ along trajectory $\boldsymbol{\xi}$, with $d_{l_{k},l_{k+1}}$ being the distance between cluster $l_k$ and $l_k+1$. Also, $l_0$ and $l_{K+1}$ represent, respectively, the DBS's original and final locations.
The available time that the DBS can use to return to origin after serving cluster $l_{k}$ on trajectory $\boldsymbol{\xi}$ is $\tau_{k}=T-\sum_{\kappa=0}^{k-1}\frac{d_{l_{\kappa},l_{\kappa+1}}}{V}-\sum_{\kappa=1}^{k}D^*_{l_{\kappa}} $.
 
  However, with one DBS serving the whole area, some users may not have the opportunity to finish their data transmission during the studied period. In such a case, the utility function of the DBS along trajectory $\boldsymbol{\xi}$ is defined as the \emph{service success rate}, which captures the fraction of users being served in the studied duration, and can be given by:

 \begin{equation}\label{eq:3}
  \setlength{\abovedisplayskip}{3 pt}
\setlength{\belowdisplayskip}{3 pt}
\begin{split}
\mu\left(\boldsymbol{\xi}\right)=\frac{\sum_{k=1}^{K}\sum_{u\in\mathcal{U}}\mathds{1}_{\left\{u\in \mathcal{U}_{l_{k}},t_u\le T-\tau_k\right\}}}{\sum_{u\in\mathcal{U}}\mathds{1}_{\left\{0\le t_u\le T\right\}}},
\end{split}
\end{equation} 
where $\mathds{1}_{\left\{x\right\}}=1$ when $x$ is true, otherwise, we have $\mathds{1}_{\left\{x\right\}}=0$. Here, $\sum_{u\in\mathcal{U}}\mathds{1}_{\left\{0\le t_u\le T\right\}}$ is the number of users request data service within the studied time duration, and $\sum_{u\in\mathcal{U}}\mathds{1}_{\left\{u\in \mathcal{U}_{l_{k}},t_u\le T-\tau_k\right\}}$ is the number of served user with cluster $l_k$. 
The value of the utility function that can be achieved with policy $\boldsymbol{\pi}$  is given by:
 \begin{equation}\label{eq:4}
  \setlength{\abovedisplayskip}{3 pt}
\setlength{\belowdisplayskip}{3 pt}
\begin{split}
&\overline \mu\left(\boldsymbol{\pi}\right)=\sum_{\boldsymbol{\xi}\in\mathcal{E}}\mu\left(\boldsymbol{\xi}\right)\prod\limits^{K-1}_{k=1}\pi\left(l_{k+1}\left| l_{k},\tau_{k} \right.\right),
\end{split}
\end{equation} 
where $\boldsymbol{\pi}$ defines the probabilities of each of DBS' moving direction at different cluster and time epoch. 
\subsection{Problem Formulation}

Our goal is to find an optimal DBS trajectory that can effectively serve all the users. 
Thus, we formulate an optimization problem whose objective is to find the policy that maximizes the expected utility of the network:
\addtocounter{equation}{0}
\begin{equation}\label{opt}
 \max_{\boldsymbol{\pi}}\overline \mu\left(\boldsymbol{\pi}\right),
\end{equation}
\vspace{-0.0cm}
\begin{align}\label{c1}
\setlength{\abovedisplayskip}{-5 pt}
&\;\;\;\;\rm{s.\;t.}\scalebox{1}{$\;\;\;\; t_{0}\left(\boldsymbol{\xi}\right)\le T,  \boldsymbol{\xi}\in \mathcal {E}, $} \tag{\theequation a}\\
&\;\;\;\;\;\;\;\;\scalebox{1}{$\;\;\;\;\;0\le \pi\left(l_{k+1}\left| l_{k},\tau_{k} \right.\right)\le 1, 1\le k\le K-1, $} \tag{\theequation b}
\end{align}
where (\ref{opt}a) means that the DBS must return to the origin before time epoch $T$. From (\ref{opt}), we can see that, the network utility is determined by the DBS's policy, which is, in turn, decided based on the quantity and occurrence of users' access requests. In essence, each realization of user requests needs a novel optimal trajectory design for the DBS.  However, the user requests in the studied network is dynamic and unpredictable. In order to provide best coverage to such users, the DBS must find an algorithm that can effectively guide it in various environments. Here, we notice that the use of optimization algorithms, such as branch and bound or nonlinear programming, is not suitable to solve (\ref{opt}), as the values of $\boldsymbol{b}$ and $\boldsymbol{t}$ follow unknown distributions. Traditional machine learning algorithms \cite{chen2019joint} are also not suitable to solve (\ref{opt}), as most of them
are overfitted to their training tasks, and, thus they are not capable of guiding the DBS in unseen tasks. 
To effectively solve (\ref{opt}), we propose a meta-reinforcement learning (MRL) algorithm that prepares the DBS for novel tasks by tuning an RL solution for each task. 
In particular, the MRL solution adaptively tunes the RL hyper-parameters so as to guide the DBS in unseen wireless environments. This, in turn, can guarantee fast convergence to maximal service success rate. 
The proposed MRL algorithm is introduced in the next section.

\section{Meta Gradient Policy Gradient Meta-Reinforcement Learning}

We now introduce an MRL algorithm, called meta-gradient policy gradient (MGPG), that merges the concept of online cross-validation \cite{xu2018meta}  with the policy gradient (PG) framework. As a classical RL algorithm, the PG algorithm finds the optimal trajectory for the DBS by running gradient descent over the policy space toward the maximal expected utility. However, this process is fitted only to certain environments, which can lead to a poor performance in dynamic networks. To address these challenges, we propose a meta-learning approach that tunes the hyper-parameters of the PG algorithm thus enabling the DBS adapt to the dynamic and unknown environments. 
Next, we first introduce the components of the proposed MGPG algorithm. Then, we explain how to use the MGPG algorithm to solve our problem.

\subsection{MGPG Components}

An MGPG algorithm consists of five components: 
\begin{enumerate}
\item {\emph{Agent}: Our agent is the DBS whose goal is to design its trajectory.}
\item {\emph{Actions}: The action of the agent is the cluster it targets at each step, the vector of actions taken by the DBS until step $k$ is denoted $\boldsymbol{a}_k=\left[a_0, a_1,\cdots,a_k\right]$ with $a_k\in\mathcal{L}$ with $a_{0}$ being the first cluster the DBS targets. }
\item {\emph{States}: Each state considers both the agent's location, represented by the cluster it currently serves, and the DBS' energy level, represented by the remaining time for the DBS to return to the origin, and it is given by $\boldsymbol{s}_k=\left[a_{k-1}, \tau_{k}\right]$, 
where $a_{k-1}$ indicates the DBS' location at step $k$. $\tau_{k}$ is the remaining time for the DBS to return to the origin at step $k$.}
\item {\emph{Policy}: The policy of the DBS is defined as the probability of choosing a given action at a given state and is denoted by $\pi_{\boldsymbol{\theta}}\left(a_k\left| {\boldsymbol{s}_k} \right.\right)$. The policy is approximated by a deep neural network parametrized by $\boldsymbol{\theta}$. This neural network takes observation on the state of the agent and outputs the probabilities of taking each action at this state. The DBS's policies at all states is denoted in vector $\boldsymbol{\pi}_{\boldsymbol{\theta}}$.} 
\item {\emph{Reward}: The reward of the DBS measures the benefit of each action. This reward is denoted as $r\left(a_k\left|\boldsymbol{s}_k\right.\right)=\frac{\sum_{u\in\mathcal{U}}\mathds{1}_{\left\{u\in \mathcal{U}_{\boldsymbol{a}_{k}}\right\}}}{\sum_{u\in\mathcal{U}}\mathds{1}_{\left\{0\le t_u\le T\right\}}}$, which is the service success rate the DBS achieves with action ${a}_{k}$ at state $\boldsymbol{s}_k$, such that $\sum_{k=1}^{K} r\left(\boldsymbol{a}_k\left|\boldsymbol{s}_k\right.\right)=\mu\left(\boldsymbol{a}_K\right)$. The average of the accumulated reward is the the expected utility or the objective function the DBS tries to optimize, that is, $\sum_{\boldsymbol{a}_K\in\mathcal{E}}\sum_{k=1}^{K} r\left(\boldsymbol{a}_k\left|\boldsymbol{s}_k\right.\right)\prod\limits^{K}_{k=1}\pi_{\boldsymbol{\theta}}\left(a_k\left| {\boldsymbol{s}_k} \right.\right)=\overline \mu\left(\boldsymbol{\pi_{\boldsymbol{\theta}}}\right)$}.
\end{enumerate}

\subsection{Policy training procedure}
In the proposed MGPG algorithm, the goal of the DBS is to find the optimal policy that maps its states to the actions leading to a maximum service success rate in various environments. The DBS will collect experiences at a realization of the user service request and then update its policy based on the utility value it obtains from these experiences.  In particular, the DBS updates it policy toward the maximized utility, i.e. in the direction to the gradient $\nabla_{\boldsymbol{\theta}}\overline \mu\left(\boldsymbol{\pi_{\boldsymbol{\theta}}}\right)$, which, based on the the policy gradient theorem, is approximated in the form of:
 \begin{equation}\label{eq:policytheorem}
  \setlength{\abovedisplayskip}{3 pt}
\setlength{\belowdisplayskip}{3 pt}
\begin{split}
\mathds{E}_{\boldsymbol{\pi}}\left(\sum_{k=1}^{K} r\left(\boldsymbol{a}_k\left|\boldsymbol{s}_k\right.\right)\left(\sum_{k=1}^{K}\nabla_{\boldsymbol{\theta}}\log \pi_{\boldsymbol{\theta}}\left(a_k\left| {\boldsymbol{s}_k} \right.\right)\right)\right).
\end{split}
\end{equation} 
The DBS, then, tunes its update direction after each policy update. Next, the procedure of the policy training and direction tuning is explained in detail.

The DBS is trained on its historical experience under a realization of the users' access requests that follow unknown distribution. The training process is given as: 
\begin{enumerate}
\item { Initialize a policy $\boldsymbol{\pi}_{\boldsymbol{\theta}^{\left(1\right)}}$ with random parameter values $\boldsymbol{\theta}^{\left(1\right)}$ and $\eta^{\left(1\right)}$. Note that, here $\boldsymbol{\theta}^{\left(1\right)}$ and $\eta^{\left(1\right)}$ are, respectively, the policy parameter and hyper-parameter at training episode $1$. }

\item { Within each training episode $i$, the DBS caries out an experience by generating a sequence of actions that are randomly selected based on policy $\boldsymbol{\pi}_{\boldsymbol{\theta}^{\left(i\right)}}$. 
 Each of such procedure generates an experience defined in a vector {\small$\boldsymbol{e}=\left[\boldsymbol{s}_{1},a_{1},r\left(a_{1}\left|\boldsymbol{s}_{1}\right.\right),\cdots, \boldsymbol{s}_{K},a_{K},r\left(a_{K}\left|\boldsymbol{s}_{K}\right.\right)\right]$} with $K$ being the terminal step. Here, $\boldsymbol{s}_{k}$ and $a_{k}$ are, respectively, the state and action at the $k$-th step of experience $\boldsymbol{e}$.}


\item { 
The discounted future reward, denoted return $G_{k}$, at each step $k$ of the DBS' experience is given by:
 \begin{equation}\label{eq:return}
  \setlength{\abovedisplayskip}{3 pt}
\setlength{\belowdisplayskip}{3 pt}
\begin{split}
G_k\left(\eta^{\left(i\right)}\right)=\sum^{K}_{\kappa=k}\left(\eta^{\left(i\right)}\right)^{\kappa-k}r\left(a_{\kappa}\left|\boldsymbol{s}_{\kappa}\right.\right),
\end{split}
\end{equation} 
where hyper-parameter $\eta^{\left(i\right)}$ is the discount factor of the return function at training episode $i$. }

\item { The DBS then updates its policy with a reinterpreted objective defined as $J\left(\eta^{\left(i\right)}, \boldsymbol{\theta}^{\left(i\right)}\right)=\sum^{K}_{k=1}G_k\left(\eta^{\left(i\right)}\right)\log\pi_{\boldsymbol{\theta}^{\left(i\right)}}\left(a_k\left| {\boldsymbol{s}_k} \right.\right)$, one the sampled experience, in particular:
 \begin{equation}\label{eq:policyupdate}
  \setlength{\abovedisplayskip}{3 pt}
\setlength{\belowdisplayskip}{3 pt}
\begin{split}
\boldsymbol{\theta}^{\left(i+1\right)}=\boldsymbol{\theta}^{\left(i\right)}+\alpha\nabla_{\boldsymbol{\theta}^{\left(i\right)}}J\left(\eta^{\left(i\right)}, \boldsymbol{\theta}^{\left(i\right)}\right),
\end{split}
\end{equation} 
where $\alpha$ is the updating step size. 
}

\item { The DBS caries out an experience $\boldsymbol{e}'$ with the updated policy $\boldsymbol{\pi}_{\boldsymbol{\theta}^{\left(i+1\right)}}$.}
 
\item {The DBS measures current hyper-parameter $\eta^{\left(i\right)}$'s performance on directing its policy to the optimal utility with the updated measurement function $\tilde J\left(\tilde\eta, \boldsymbol{\theta}^{\left(i+1\right)}\right)$ with a fixed hyper-parameter $\tilde\eta$. The update function is given as:

\begin{equation}\label{eq:hyperupdate}
  \setlength{\abovedisplayskip}{3 pt}
\setlength{\belowdisplayskip}{3 pt}
\begin{split}
\eta^{\left(i+1\right)}=\eta^{\left(i\right)}-\beta\nabla_{\eta^{\left(i\right)}}\tilde J\left(\tilde \eta, \boldsymbol{\theta}^{\left(i+1\right)}\right).
\end{split}
\end{equation}  }
\end{enumerate}
Steps 2 through 6 are repeated until training is complete and yields an optimal policy $\boldsymbol{\pi}_{\boldsymbol{\theta}^{*}}$. Note that, instead of performing the expensive deviation calculations $\nabla_{\boldsymbol{\theta}^{\left(i\right)}}J\left(\eta^{\left(i\right)}, \boldsymbol{\theta}^{\left(i\right)}\right)$, $\nabla_{\eta^{\left(i\right)}}\tilde J\left(\tilde \eta, \boldsymbol{\theta}^{\left(i+1\right)}\right)$, the DBS can approximate the update and tuning direction based on Proposition 1.

\begin{proposition}\label{pp1}\emph{
At each iteration $i$ of the policy update procedure, the DBS calculates the update direction using: 
 \begin{equation}\label{eq:step1}
\begin{split}
&\nabla_{\boldsymbol{\theta}^{\left(i\right)}}J\left(\eta^{\left(i\right)}, \boldsymbol{\theta}^{\left(i\right)}\right) \\
&= \sum^{K}_{k=1}\sum^{K}_{\kappa=k}\left(\eta^{\left(i\right)}\right)^{\kappa-k}r\left(a_{\kappa}\left|\boldsymbol{s}_{\kappa}\right.\right)\frac{\nabla_{\boldsymbol{\theta}^{\left(i\right)}}\pi_{\boldsymbol{\theta}^{\left(i\right)}}\left(a_k\left| {\boldsymbol{s}_k} \right.\right)}{\pi_{\boldsymbol{\theta}^{\left(i\right)}}\left(a_k\left| {\boldsymbol{s}_k} \right.\right)}.
\end{split}
\end{equation}  
After every policy update, the DBS estimate the hyper-parameter tuning direction $\nabla_{\eta^{\left(i\right)}}\tilde J\left(\tilde \eta, \boldsymbol{\theta}^{\left(i+1\right)}\right)$, with $\tilde \eta$ being $1$, in the form of: 
\begin{equation}\label{eq:step2}
\begin{split}
&\nabla_{\eta^{\left(i\right)}}\tilde J\left(\tilde \eta, \boldsymbol{\theta}^{\left(i+1\right)}\right) = \alpha\sum^{K}_{k'=1}\sum^{K}_{k=1}A_{k'}B_{k}x_{k'}y_{k}.
\end{split}
\end{equation} 
where $A_{k'}=\sum^{K}_{\kappa=k'}r\left(a_{\kappa}\left|\boldsymbol{s}_{\kappa}\right.\right)$, $x_{k'}=\frac{\nabla_{\boldsymbol{\theta}^{\left(i+1\right)}}\pi_{\boldsymbol{\theta}^{\left(i+1\right)}}\left(a_k\left| {\boldsymbol{s}_k} \right.\right)}{\pi_{\boldsymbol{\theta}^{\left(i+1\right)}}\left(a_k\left| {\boldsymbol{s}_k} \right.\right)}$, $B_{k}=\sum^{K}_{\kappa=1}\left(\kappa-k\right)\left(\eta^{\left(i\right)}\right)^{\kappa-k-1}r\left(a_{\kappa}\left|\boldsymbol{s}_{\kappa}\right.\right)$, and $y_{k}=\frac{\nabla_{\boldsymbol{\theta}^{\left(i\right)}}\pi_{\boldsymbol{\theta}^{\left(i\right)}}\left(a_k\left| {\boldsymbol{s}_k} \right.\right)}{\pi_{\boldsymbol{\theta}^{\left(i\right)}}\left(a_k\left| {\boldsymbol{s}_k} \right.\right)}$}.
\end{proposition}
\begin{proof}
See Appendix A.
\end{proof}

As such, using Proposition 1, the DBS can estimate the update and tuning direction simply based on the collected rewards and policy deviation $\frac{\nabla_{\boldsymbol{\theta}}\pi_{\boldsymbol{\theta}}\left(a_k\left| {\boldsymbol{s}_k} \right.\right)}{\pi_{\boldsymbol{\theta}}\left(a_k\left| {\boldsymbol{s}_k} \right.\right)}$. 
The mini-batch training procedure of the MGPG based solution is shown in Algorithm 1. With this mini-batch training procedure, the DBS updates its policy over a whole experience to reduce the variance caused by the action sampling without the need to store a big dataset.
At the beginning of the algorithm, the DBS randomly chooses its policy parameter $\boldsymbol{\theta}^{\left(1\right)}$ and hyper-parameter $\eta^{\left(1\right)}$. The DBS, then, carries out an experience by selecting a trajectory with policy $\pi_{\boldsymbol{\theta}^{\left(1\right)}}$. The DBS, then, flies along the selecting trajectory to serve the users. The service success rate resulted from the selected trajectory is recorded by the DBS. After returning to the origin, the DBS updates its policy parameter based on (\ref{eq:policyupdate}), with the service success rate it obtains from the selected trajectory. This update is then evaluated with (\ref{eq:step2}) on an experience collected with the updated policy. The hyper-parameter is tuned with function (\ref{eq:hyperupdate}) based on such evaluation. This policy training procedure is repeated by the DBS flying around the network until a convergence is reached. In essence, the DBS updates its policy episodically with a formerly generated experience at its origin, and tunes its hyper-parameter online right after each policy update. The complexity of the the MGPG based solution is $\mathcal{O}\left(\left(\upsilon+1\right) n C \right)$, where $\upsilon$ is the iteration at which the convergence is reached, $n$ is the number of elements within policy parameter $\boldsymbol{\theta}$ and $C$ is the time complexity of calculating the gradient of each element (i.e. each element in $\boldsymbol{\theta}$). This complexity is considerably low, as it is close to the one of the vanilla policy gradient algorithm, $\mathcal{O}\left(\upsilon n C \right)$. Other advanced PG solutions such as advantage actor-critic (A2C),  deep deterministic policy gradient (DDPG), however, have a huge time complexity, since each neural network they adopt have a time complexity of $\mathcal{O}\left(\upsilon n^5\right)$. In the studied drone-aided network, the lightweight MGPG algorithm is more desirable and practical, as the computation capacity of a drone is limited.




\begin{algorithm}[t]\footnotesize
\caption{Proposed MGPG algorithm for trajectory design. }   
\label{table:param}   
\setlength{\abovecaptionskip}{-15pt} 
\setlength{\belowcaptionskip}{-15pt}
\begin{algorithmic} [1] 
\REQUIRE The user locations, time constraints. \\ 
\vspace{2pt}  
\ENSURE Initialize parameter $\boldsymbol{\theta}^{\left(1\right)}$ and $\eta^{\left(1\right)}$. \\

\vspace{2pt}  
\FOR {Policy training epoch $i =1:I$} 
\vspace{2pt}  
\STATE Caries out an experience $\boldsymbol{e}$.
\vspace{2pt}  
\STATE Update policy parameter based on (\ref{eq:step1}).  
\vspace{2pt}  
\STATE Caries out an experience $\boldsymbol{e}'$ with updated policy.
\vspace{2pt}  
\STATE Update hyper-parameter based on (\ref{eq:step2}).  
\vspace{2pt}  
\ENDFOR  
\end{algorithmic}
\end{algorithm} 

%
\section{Simulation Results and Analysis}

For our simulations, we consider a scenario with one DBS serving $U=100$ mobile users. 
Other parameters used in the simulations are listed in Table \ref{table:param}. The proposed MGPG algorithm results are compared to the vanilla policy gradient algorithm\cite{sutton2000policy}, called policy gradient algorithm hereinafter. All statistical results are averaged over a large number of independent runs.

%
%
%
%
%
\begin{table}
 \vspace{-0.32cm}
  \newcommand{\tabincell}[2]{\begin{tabular}{@{}#1@{}}#2\end{tabular}}
\renewcommand\arraystretch{1}
 \caption{
    \vspace*{-0.2em}Simulation Parameters\cite{840210}}\vspace*{-1em}
\centering  
\begin{tabular}{|c|c|c|c|}
\hline
\textbf{Parameter} & \textbf{Value} & \textbf{Parameter} & \textbf{Value} \\
\hline
$P_u $ & 20 dBm & $\alpha$ & 2\\
\hline
$\sigma^2$ & -104 dB & $B $ & 20 MHz\\
\hline
$V$ & 30 m/s & $m$ & 3\\
\hline

\end{tabular}
 \vspace{-0.15cm}
\end{table}  

\begin{figure}[!t]
  \begin{center}
   \vspace{0cm}
    \includegraphics[width=7.12cm]{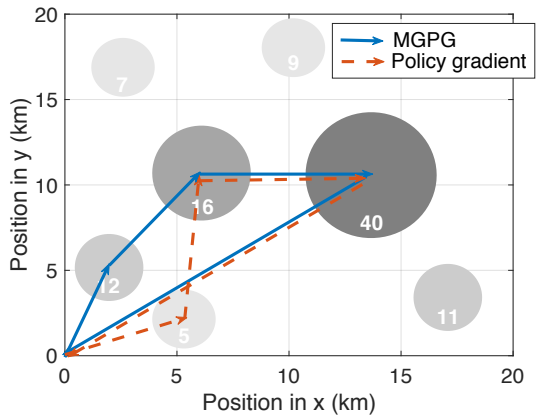}
    \vspace{-0.2cm}
    \caption{\label{fig:rst} Snapshot of trajectories resulting from all considered algorithms in a sample unseen environment.}
  \end{center}
  \vspace{-0.7cm} 
\end{figure}
Fig. \ref{fig:rst} shows a snapshot of the trajectories resulting from the proposed MGPG and the policy gradient algorithm in an illustrative unseen realization of user requests. The grey circles are the user clusters labeled by the number of active users. In this figure, we can see that, facing an unseen task with its manually pre-defined hyper-parameters, the policy gradient algorithm leads the DBS to only a suboptimal solution. The MGPG algorithm, however, effectively finds the optimal trajectory for the DBS by tuning the hyper-parameters in this unseen environment.

\begin{figure}[!t]
  \begin{center}
   \vspace{0cm}
    \includegraphics[width=7.2cm]{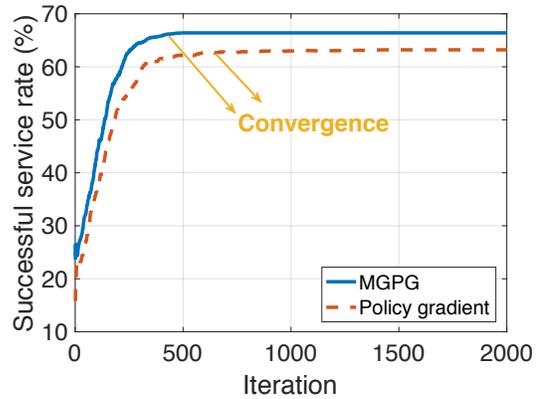}
    \vspace{-0.2cm}
    \caption{\label{fig:cvg} Convergence of all considered algorithms.}
  \end{center}
  \vspace{-0.7cm} 
\end{figure}

Fig. \ref{fig:cvg} shows the convergence of the proposed MGPG algorithm. From Fig. \ref{fig:cvg}, we observe that the MGPG algorithm requires approximately $450$ iterations to reach convergence, which is $25\%$ less than the number of iterations
required for convergence of the policy gradient algorithm. This is because the update direction of MGPG keeps being tuned in the policy update process, and, thus, it has a smaller angle towards the optimal result. 
than the sub-gradient algorithm. Fig. \ref{fig:cvg} also shows that the proposed MGPG algorithm achieves a service success rate of $68\%$ at its convergence, which is about $10\%$ higher the one reached by the policy gradient baseline. This stems from the
fact that the proposed MGPG algorithm finds the hyper-parameter results in the best performance of the RL algorithm, as the MGPG algorithm keeps tuning the hyper-parameters in policy update process.

\begin{figure}[t]
  \begin{center}
   \vspace{0cm}
    \includegraphics[width=7.2cm]{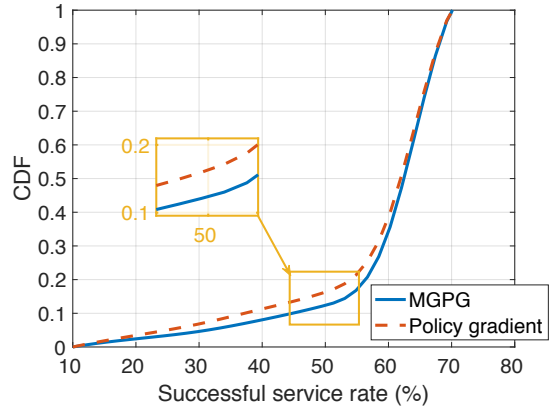}
    \vspace{-0.2cm}
    \caption{\label{fig:CDF}Cumulative distribution function (CDF) in terms of service success rate. }
  \end{center}\vspace{-0.7cm}
\end{figure}

Fig. \ref{fig:CDF} shows the cumulative distribution function (CDF) of the service success rate at the convergence of all independent runs. For example, from Fig. \ref{fig:CDF}, we can see that the proposed algorithm converges to a service success rate lower than $50\%$ within only about $16\%$ of the independent runs.
In Fig. \ref{fig:CDF}, we can see that the proposed algorithm achieves about $27\%$ gain in the frequency of convergence at service success rate over $50\%$, compared to the policy gradient algorithm. This means that the proposed algorithm is more likely to converge to a higher service success rate, that is, it reduces the variance of algorithm performance. The main reason is that the proposed algorithm tunes the hyper-parameter toward a return $G_t$ resulting in the best performance of the DBS. 
\section{Conclusion}
In this paper, we have studied a challenging trajectory design problem for UAVs in dynamic and unknown wireless network environments. We formulated an optimization problem that seeks
to maximize the DBS service success rate. To solve this problem, we have developed a novel meta-reinforcement learning based algorithm that enables the DBS to adapt to unseen environments quickly. Simulation results have shown that, by tuning the hyper-parameters in RL, the proposed MGPG solution can yield a $25\%$ improvement in the convergence speed, and an about $10\%$ improvement in the final service success rate, compared to the vanilla policy gradient algorithm. 
Future work can consider various extensions such as the cases with multiple DBSs and other dynamic factors.

%
\section*{Appendix}
 {
\subsection{Proof of Proposition \ref{pp1}}\label{Bp:b}  
\begin{proof} 
{Based on (\ref{eq:return}), the policy update direction at iteration $i$ can be given as:
 \begin{equation}\label{eq:pf1}
\begin{split}
&\nabla_{\boldsymbol{\theta}^{\left(i\right)}}J\left(\eta^{\left(i\right)}, \boldsymbol{\theta}^{\left(i\right)}\right) \\
&=\sum^{K}_{k=1}\sum^{K}_{\kappa=k}\left(\eta^{\left(i\right)}\right)^{\kappa-k}r\left(a_{\kappa}\left|\boldsymbol{s}_{\kappa}\right.\right)\nabla_{\boldsymbol{\theta}^{\left(i\right)}}\log\pi_{\boldsymbol{\theta}^{\left(i\right)}}\left(a_k\left| {\boldsymbol{s}_k} \right.\right)\\
&= \sum^{K}_{k=1}\sum^{K}_{\kappa=k}\left(\eta^{\left(i\right)}\right)^{\kappa-k}r\left(a_{\kappa}\left|\boldsymbol{s}_{\kappa}\right.\right)\frac{\nabla_{\boldsymbol{\theta}^{\left(i\right)}}\pi_{\boldsymbol{\theta}^{\left(i\right)}}\left(a_k\left| {\boldsymbol{s}_k} \right.\right)}{\pi_{\boldsymbol{\theta}^{\left(i\right)}}\left(a_k\left| {\boldsymbol{s}_k} \right.\right)}.
\end{split}
\end{equation} 
Also, to tune hyper-parameters $\eta$ in the direction that achieves the best service success rate in the studied network, we measure $\eta$ by cross validating the updated parameter $\boldsymbol{\theta}^{\left(i+1\right)}$ on second trajectory $\boldsymbol{e}'$. In particular, the tuning direction at step $i$ is given in the form of:
 \begin{equation}\label{eq:pf2}
\setlength{\abovedisplayskip}{2 pt}
\setlength{\belowdisplayskip}{2 pt}
\begin{split}
{\nabla_{\eta^{\left(i\right)}}\tilde J\left(\tilde \eta, \boldsymbol{\theta}^{\left(i+1\right)}\right)=\frac{{\partial \tilde J\left(\tilde\eta, \boldsymbol{\theta}^{\left(i+1\right)}\right)}}{{\partial  \boldsymbol{\theta}^{\left(i+1\right)}}}\frac{{d\boldsymbol{\theta}^{\left(i+1\right)}}}{{d\eta^{\left(i\right)}}},}
\end{split}
\end{equation} 
where:
 \begin{equation}\label{eq:pf3}
\begin{split}
{\frac{{\partial \tilde J\left(\tilde\eta, \boldsymbol{\theta}^{\left(i+1\right)}\right)}}{{\partial  \boldsymbol{\theta}^{\left(i+1\right)}}}=\sum^{K}_{k=1}\sum^{K}_{\kappa=k}r\left(a_{\kappa}\left|\boldsymbol{s}_{\kappa}\right.\right)\frac{\nabla_{\boldsymbol{\theta}^{\left(i+1\right)}}\pi_{\boldsymbol{\theta}^{\left(i+1\right)}}\left(a_k\left| {\boldsymbol{s}_k} \right.\right)}{\pi_{\boldsymbol{\theta}^{\left(i+1\right)}}\left(a_k\left| {\boldsymbol{s}_k} \right.\right)},}
\end{split}
\end{equation} 
Meanwhile:
 \begin{equation}\label{eq:pf4}\small
\begin{split}
&\frac{{d\boldsymbol{\theta}^{\left(i+1\right)}}}{{d\eta^{\left(i\right)}}}=\frac{d\left[\boldsymbol{\theta}^{\left(i\right)}+ \alpha\nabla_{\boldsymbol{\theta}^{\left(i\right)}}J\left(\eta^{\left(i\right)}, \boldsymbol{\theta}^{\left(i\right)}\right)\right]}{d\eta^{\left(i\right)}}\\
&=\frac{d\boldsymbol{\theta}^{\left(i\right)}}{d\eta^{\left(i\right)}}+\frac{\alpha d\sum^{K}_{k=1}\!\!\sum^{K}_{\kappa=k}\!\!\left(\eta^{\left(i\right)}\right)^{\kappa-k}r\left(a_{\kappa}\left|\boldsymbol{s}_{\kappa}\right.\right)\frac{\nabla_{\boldsymbol{\theta}^{\left(i\right)}}\pi_{\boldsymbol{\theta}^{\left(i\right)}}\left(a_k\left| {\boldsymbol{s}_k} \right.\right)}{\pi_{\boldsymbol{\theta}^{\left(i\right)}}\left(a_k\left| {\boldsymbol{s}_k} \right.\right)}}{d\eta^{\left(i\right)}}\\
&=\frac{d\boldsymbol{\theta}^{\left(i\right)}}{d\eta^{\left(i\right)}}\!+\!\alpha\!\!\sum^{K}_{k=1}\!\!\sum^{K}_{\kappa=k}\!\!\left(\kappa-k\right)\!\left(\eta^{\left(i\right)}\!\right)^{\!\kappa-k-1\!}\!\!\!r\left(a_{\kappa}\!\left|\boldsymbol{s}_{\kappa}\right.\!\right)\!\frac{\nabla_{\boldsymbol{\theta}^{\left(i\right)}}\pi_{\boldsymbol{\theta}^{\left(i\right)}}\left(a_k\left| {\boldsymbol{s}_k} \right.\right)}{\pi_{\boldsymbol{\theta}^{\left(i\right)}}\left(a_k\left| {\boldsymbol{s}_k} \right.\right)}.
\end{split}
\end{equation} 
Here, we decay the first term in (\ref{eq:pf4}), to effectively tuning $\eta$\cite{817990}. In such a case, the first term is rewritten as $\mu\frac{d\boldsymbol{\theta}^{\left(i\right)}}{d\eta^{\left(i\right)}}$, where $\mu\in\left[0,1\right]$ is the decay factor. In practice, we approximation the first term by setting $\mu=0$ to reduce computation complexity. In essence, we only consider hyper-parameter $\eta$'s effect on a single policy update. In summary, the tuning direction is finally approximated as:
 \begin{equation}\label{eq:pf5}\small
\setlength{\abovedisplayskip}{1 pt}
\setlength{\belowdisplayskip}{1 pt}
\begin{split}
&\nabla_{\eta^{\left(i\right)}}\tilde J\left(\tilde \eta, \boldsymbol{\theta}^{\left(i+1\right)}\right) \approx  \alpha\sum^{K}_{k'=1}\sum^{K}_{k=1}A_{k'}B_{k}x_{k'}y_{k},
\end{split}
\end{equation} 
where $A_{k'}=\sum^{K}_{\kappa=k'}r\left(a_{\kappa}\left|\boldsymbol{s}_{\kappa}\right.\right)$, $x_{k'}=\frac{\nabla_{\boldsymbol{\theta}^{\left(i+1\right)}}\pi_{\boldsymbol{\theta}^{\left(i+1\right)}}\left(a_k\left| {\boldsymbol{s}_k} \right.\right)}{\pi_{\boldsymbol{\theta}^{\left(i+1\right)}}\left(a_k\left| {\boldsymbol{s}_k} \right.\right)}$, $B_{k}=\sum^{K}_{\kappa=1}\left(\kappa-k\right)\left(\eta^{\left(i\right)}\right)^{\kappa-k-1}r\left(a_{\kappa}\left|\boldsymbol{s}_{\kappa}\right.\right)$, and $y_{k}=\frac{\nabla_{\boldsymbol{\theta}^{\left(i\right)}}\pi_{\boldsymbol{\theta}^{\left(i\right)}}\left(a_k\left| {\boldsymbol{s}_k} \right.\right)}{\pi_{\boldsymbol{\theta}^{\left(i\right)}}\left(a_k\left| {\boldsymbol{s}_k} \right.\right)}$. This completes the proof.
}
\end{proof}}
 
\bibliographystyle{IEEEbib}
\def\baselinestretch{1.07}
\bibliography{references}
\end{document}